%% file: icdm.tex
\documentclass[10pt,twocolumn,english]{IEEEtran}
\usepackage[T1]{fontenc}
\usepackage[latin9]{inputenc}
\usepackage{babel}
\usepackage{url}
\usepackage{algorithm2e}
\usepackage{amsmath}
\usepackage{amsthm}
\usepackage{amssymb}
\usepackage{graphicx}
\usepackage[unicode=true,pdfusetitle,
 bookmarks=true,bookmarksnumbered=false,bookmarksopen=false,
 breaklinks=false,pdfborder={0 0 1},backref=false,colorlinks=false]
 {hyperref}

\makeatletter
\theoremstyle{plain}
\ifx\thechapter\undefined
	\newtheorem{thm}{\protect\theoremname}
\else
	\newtheorem{thm}{\protect\theoremname}[chapter]
\fi
\theoremstyle{plain}
\newtheorem{lem}[thm]{\protect\lemmaname}

\usepackage{graphicx,epstopdf}
\IEEEoverridecommandlockouts

\usepackage{amsmath}
\usepackage{balance}
\usepackage{fancyhdr}

\newcommand{\indep}{\perp \!\!\! \perp}

\newcommand{\argmax}{\arg\,\max}

\makeatother

\providecommand{\lemmaname}{Lemma}
\providecommand{\theoremname}{Theorem}

\begin{document}
\title{Conditional Independence Testing via\\Latent Representation Learning}
\author{\textbf{Bao Duong, Thin Nguyen}\\
Applied Artificial Intelligence Institute (A\textsuperscript{2}I\textsuperscript{2}),
Deakin University, Australia\\
Email\emph{: \{duongng,thin.nguyen\}@deakin.edu.au}}
\maketitle
\begin{abstract}
\input{abstract.tex}
\end{abstract}


\section{Introduction}

\input{intro.tex}

\section{Related Works\label{sec:Background}}

\input{related.tex}

\section{Latent representation based Conditional Independence Testing\label{sec:Methodology}}

\input{methodology.tex}

\section{Experiments\label{sec:Experiments}}

\input{experiments.tex}

\section{Conclusion and Future Work\label{sec:Conclusion-and-Future}}

\input{conclusion.tex}\balance

\bibliographystyle{plain}
\bibliography{icdm}

\end{document}

%% file: abstract.tex
Detecting conditional independencies plays a key role in several statistical
and machine learning tasks, especially in causal discovery algorithms.
In this study, we introduce \textbf{LCIT} (\textbf{L}atent representation
based \textbf{C}onditional \textbf{I}ndependence \textbf{T}est)--a
novel non-parametric method for conditional independence testing based
on representation learning. Our main contribution involves proposing
a generative framework in which to test for the independence between
$X$ and $Y$ given $Z$, we first learn to infer the latent representations
of target variables $X$ and $Y$ that contain no information about
the conditioning variable $Z$. The latent variables are then investigated
for any significant remaining dependencies, which can be performed
using the conventional partial correlation test. The empirical evaluations
show that \textbf{LCIT} outperforms several state-of-the-art baselines
consistently under different evaluation metrics, and is able to adapt
really well to both non-linear and high-dimensional settings on a
diverse collection of synthetic and real data sets.

%% file: intro.tex
\begin{figure*}[t]

\centering{}\includegraphics[width=1\textwidth]{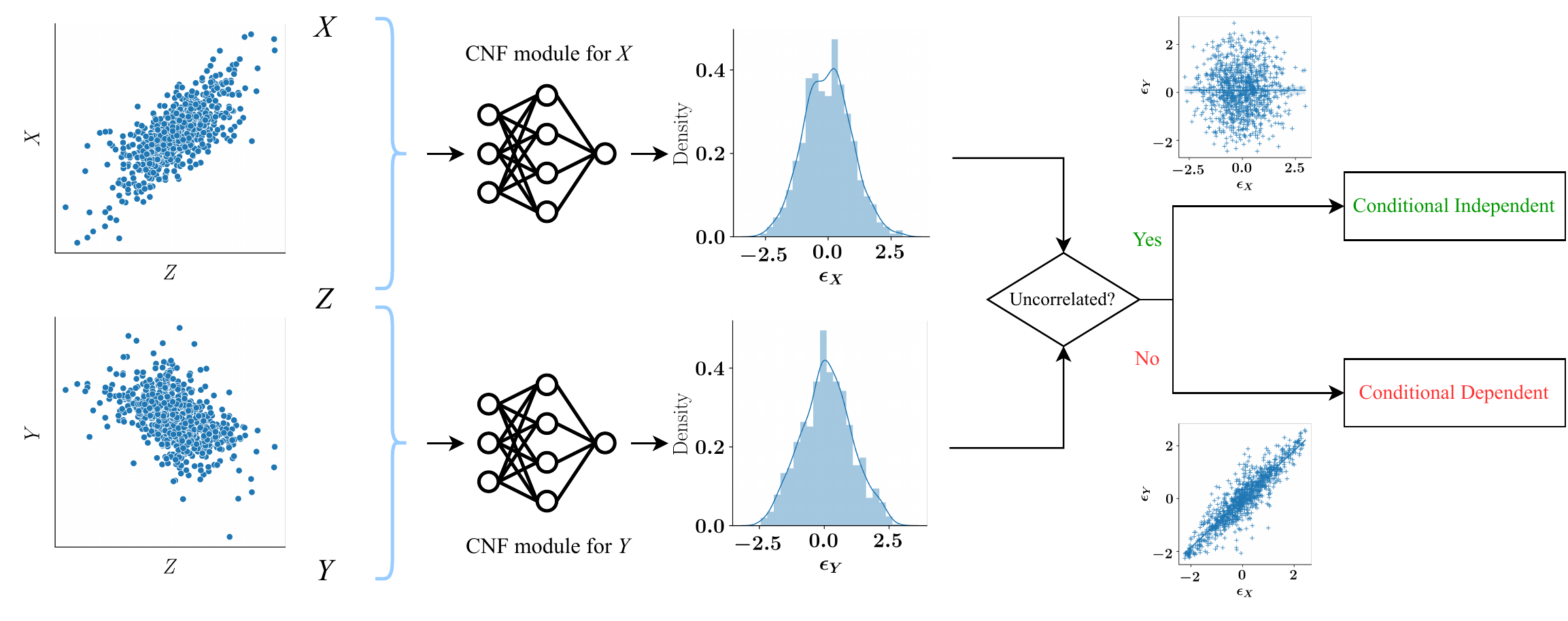}\caption{The proposed \textbf{L}atent based \textbf{C}onditional \textbf{I}ndependence
\textbf{T}est (\textbf{LCIT}) framework. First, $X$ and $Y$ are
transformed into respective latent spaces using two Conditional Normalizing
Flows (CNF) modules independently learned from samples of $\left(X,Z\right)$
and $\left(Y,Z\right)$, respectively. The latent variables ($\epsilon_{X}$
and $\epsilon_{Y}$), which have standard Gaussian distributions by
design, are then used as inputs for the conventional correlation test.
If $\epsilon_{X}$ and $\epsilon_{Y}$ are indeed uncorrelated then
we accept the null hypothesis ($\mathcal{H}_{0}:X\indep Y|Z$), otherwise
we reject the null hypothesis and accept the alternative ($\mathcal{H}_{1}:X\protect\not\indep Y|Z$).\label{fig:illustration}}
\end{figure*}

Conditional independence (CI) tests concern with the problem of testing
if two random variables $X$ and $Y$ are statistically independent
after removing all the influences coming from the conditioning variables
$Z$ (denoted as $X\indep Y|Z$), using the empirical samples from
their joint distribution $p(X,Y,Z)$. More formally, we consider the
hypothesis testing setting with:

\begin{align*}
\text{Null hypothesis }\mathcal{H}_{0} & :X\indep Y|Z\\
\text{Alternative hypothesis }\mathcal{H}_{1} & :X\not\indep Y|Z
\end{align*}

With this functionality, CI tests have been extensively leveraged
as the backbone of causal discovery algorithms which aim to disassemble
the causal interrelationships embeded in the observed data. More specifically,
in constraint-based causal discovery methods such as the PC algorithm
and its variants \cite{spirtes1991algorithm,spirtes2000causation},
CI tests are used to detect if each pair of variables has an intrinsic
relationship that cannot be explained by any other intermediate variables,
and so forth connectivity of those who do not share this kind of relation
are progressively removed. The outputs from such methods are extremely
valuable in many scientific sectors such as econometrics \cite{hunermund2019causal},
neuroscience \cite{cao2019causal}, machine learning \cite{scholkopf2016causal,scholkopf2021toward,scholkopf2022statistical},
and especially bioinformatics \cite{sachs2005causal,zhang2013integrated}
where the behavioral links between genes, mutations, diseases, etc.
are in the heart of curiosity.

Here we consider the continuous instance of the problem where $X$,
$Y$, and $Z$ are real-valued random vectors or variables, which
is significantly harder than the discrete case in general. This is
because discrete probabilistic quantities are usually more tractable
to compute, in contrary with their continuous counterparts. In fact,
many methods must resort to binning continuous variables into discrete
values for their tests \cite{su2008nonparametric,diakonikolas2016new,warren2021wasserstein}.
However, this technique is prone to loss of information and is erroneous
in high dimensions due to the curse of dimensionality. This also highlights
the inherent difficulty of CI testing in continuous domain.

From the technical perspective, CI tests can be roughly categorized
into four major groups based on their conceptual essence, including
distance-based methods, kernel-based methods, regression-based methods,
and model-based methods. In distance-based methods \cite{su2008nonparametric,etesami2017new,warren2021wasserstein},
the direct characterization of CI, $p\left(x|z\right)p\left(y|z\right)=p\left(x,y|z\right)$
or $p(x|y,z)=p(x|z)$, is exploited and methods in this class aim
to explicitly measure the difference between respective probability
densities or distributions. These methods usually employ discretization,
which has been discussed to be faulty as the data dimensionality increases.

Next, kernel-based methods \cite{fukumizu2007kernel,zhang2012kernel,doran2014permutation,strobl2019approximate}
adopt kernels to exploit high order properties of the distributions
via sample-wise similarities in higher spaces. More concretely, variables
are mapped into reproducing kernel Hilbert spaces (RKHS), where their
conditional independence can be reflected by the partial uncorrelatedness
of functions in the RKHS. However, as noted in \cite{ramdas2015decreasing},
the performance of kernel-based methods may polynomially deteriorate
w.r.t. the dimensionality of data as well.

The next group of CI tests is regression-based methods \cite{grosse2016identification,zhang2017causal,zhang2021testing,zhang2022residual},
which transform the original CI testing problem into a more manageable
problem of marginal independence testing between the regression residuals.
Typically, these methods suppose that $Z$ is a confounder set of
$X$ and $Y$, meaning it causes both $X$ and $Y$ in the underlying
generation process, and the generating functions are additive noise
models (ANMs), so that the residuals can be fully retrieved without
any remaining information from $Z$ by using suitable regression functions.
While being simple, the application of regression-based methods is
quite limited since in general there is a high chance that $Z$ is
not a cause of $X$ or $Y$, or the generating mechanisms may have
non-additive noises.

Finally, the last group of CI methods contain a diverse set of approaches
leveraging learning algorithms as the main technical device. For example,
in \cite{runge2018conditional,mukherjee2020ccmi,kubkowski2021gain}
the conditional mutual information (CMI) between $X$ and $Y$ given
$Z$ is considered as the test statistic and estimated using $k$-nearest
neighbors estimators. Additionally, methods utilizing generative adversarial
neural networks (GAN) \cite{bellot2019conditional,shi2021double}
have also been proposed where new samples are generated to help simulate
the null distributions of the test statistics. Also, in \cite{sen2017model}
the use of supervised classifiers is harnessed to differentiate between
conditional independence and conditional dependence samples.

\textbf{Present~study.}\quad{}In this study we offer a novel approach
to test for conditional independence motivated by regression-based
methods and executed via latent representation learning with generative
models. Our proposed method, called \textbf{L}atent representation
based \textbf{C}onditional \textbf{I}ndependence \textbf{T}esting
(\textbf{LCIT})\footnote{Source code and related data are available at \url{https://github.com/baosws/LCIT}},
learns to infer the latent representations of $X$ and $Y$ conditioned
on $Z$, then tests for their unconditional independence.

To infer the latent variable, we make use of Normalizing Flows (NF)
\cite{tabak2010density}, which is a subset of generative modeling
methods capable of transforming any simple distribution into much
more complex distributions via a sequence of bijective maps. At the
same time, we can also invert variables with highly complex densities
into more manageable distributions through the inverse mode of the
flows, thanks to the bijective maps.

In comparison with other branches of generative methods, while GANs
\cite{goodfellow2014generative} are able to generate high quality
samples but do not offer the ability to infer latent representations;
and Variational Auto-encoders (VAEs) \cite{kingma2013auto} can both
generate new samples as well as infer latent variables but the information
may be loss due to non-vanishing reconstruction errors; the bijective
map in NF is perfectly fit to our methodological design due to its
ability to infer latent representations without loss of information.
The importance of the information preservation will become apparent
later in Section\textbf{~}\ref{sec:Methodology}. Moreover, since
NF can be parametrized with neural networks, we can benefit vastly
from their expressiveness that allows them to learn any distributions
with arbitrary precision.

In addition, while many approaches in the literature of conditional
independence testing require bootstrap or permutation to estimate
the $p$-value since their test statistics have non-trivial null distributions,
e.g., \cite{zhang2012kernel,runge2018conditional,bellot2019conditional,mukherjee2020ccmi,shi2021double,zhang2021testing,zhang2022residual},
\textbf{LCIT} is able to estimate the approximate $p$-value without
data resampling due to the latent variables being jointly Gaussian.
We demonstrate the effectiveness of our method in conditional independence
testing with extensive numerical simulations, as well as in real datasets.
The empirical results show that \textbf{LCIT} outperforms existing
state-of-the-art methods in a consistent manner.

\textbf{Contributions.}\quad{}This paper offers three key contributions
summarized as follows:
\begin{enumerate}
\item We present a conversation of conditional independence into unconditional
independence of ``latent'' variables via invertible transformations
of target variables, adjusted on the conditioning set. This characterization
allows us to operate with the conventional marginal independence testing
problem which is less challenging than the original conditional problem.
\item We introduce a new conditional independence testing algorithm, called
\textbf{L}atent representation based \textbf{C}onditional \textbf{I}ndependence
\textbf{T}esting (\textbf{LCIT}), which harnesses a Conditional Normalizing
Flows framework to convert target variables into corresponding latent
representations, where their uncorrelatedness entails conditional
independence in the original space. To the best of our knowledge,
this is the first time Normalizing Flows is applied into the problem
of conditional independence testing. Additionally, deviating from
many of the existing non-parametric tests, \textbf{LCIT} can estimate
exactly the $p$-value which is computationally cheaper than methods
involving bootstrapping or permutation. See Figure~\ref{fig:illustration}
for an illustration of our framework. 
\item We demonstrate the effectiveness of the proposed \textbf{LCIT} method
with a comparison against various state-of-the-art baselines on both
synthetic and real datasets, where \textbf{LCIT} is showed to outperform
other methods in several metrics.
\end{enumerate}
\textbf{Paper~organization.}\quad{}In the following parts of the
paper, we first briefly highlight major state-of-the-art ideas in
the literature of conditional independence hypothesis testing in Section\textbf{~}\ref{sec:Background}.
Then, in Section\textbf{~}\ref{sec:Methodology} we describe in detail
our proposed \textbf{LCIT} method for conditional independence testing.
Next, in Section\textbf{~}\ref{sec:Experiments} we perform experiments
to demonstrate the strength of \textbf{LCIT} in both synthetic and
real data settings. Finally, the paper is concluded with a summary
and suggestions for future developments in Section\textbf{~}\ref{sec:Conclusion-and-Future}.

%% file: related.tex
\inputencoding{latin9}Based on the technical device, CI test designs
can be practically clustered into four main groups: distance-based
methods \cite{su2008nonparametric,etesami2017new,warren2021wasserstein},
kernel-based methods \cite{fukumizu2007kernel,zhang2012kernel,doran2014permutation,strobl2019approximate},
regression-based methods \cite{zhang2022residual}, and lastly model-based
methods \cite{runge2018conditional,bellot2019conditional,mukherjee2020ccmi,shi2021double}.

\subsection{Distance-based~methods}

Starting from the most common definition of conditional independence--$X\indep Y|Z$
if and only if $p\left(x|z\right)p\left(y|z\right)=p\left(x,y|z\right)$
(or equivalently, $p(x|y,z)=p(x|z)$) for every realizations $y$
and $z$ of $Y$ and $Z$, respectively--early methods set the first
building blocks by explicitly estimating and comparing the relevant
densities or distributions.

For example, \cite{su2008nonparametric} measure the distance between
two conditionals $p\left(x|y,z\right)$ and $p\left(x|z\right)$ using
the weighted Hellinger distance. Following this direction, \cite{etesami2017new}
devise a new conditional dependence measure equal to the supremum
of the Wasserstein distance between $p(X|y,z)$ and $p(X|y',z)$ over
all realizations $y,y',z$.

Additionally, in \cite{warren2021wasserstein} the Wasserstein distance
between two conditional distributions $p\left(X,Y|z\right)$ and $p\left(X|z\right)\otimes p\left(Y|z\right)$
is measured for each discretized value of $z$.

\subsection{Kernel-based~methods}

Many other approaches avoid the difficulties in estimating conditional
densities with alternative characterizations of conditional independence.
Particularly, when $X,Y,Z$ are jointly multivariate normal then the
conditional independence $X\indep Y|Z$ reduces to the vanish of the
partial correlation coefficient $\rho_{XY\cdot Z}$ \cite{baba2004partial},
which is easy to test for zero since its Fisher transformation follows
an approximately normal distribution under the null \cite{hotelling1953new}.

Departing from that, a large body of works has focused on kernel methods
\cite{fukumizu2007kernel,zhang2012kernel,doran2014permutation,strobl2019approximate},
which can be perceived as a non-parametric generalization of the connection
between partial uncorrelatedness and conditional independence of Gaussian
variables \cite{zhang2012kernel}. These methods follow the CI characterization
established in \cite{fukumizu2004dimensionality} where the conditional
independence is expressed in terms of conditional cross-covariance
operators of functions in the reproducing kernel Hilbert spaces (RKHS):
$\Sigma_{YX|Z}:=\Sigma_{YX}-\Sigma_{YZ}\Sigma_{ZZ}^{-1}\Sigma_{ZX}$
where the cross-covariance operator is defined as $\left\langle g,\Sigma_{YX}f\right\rangle :=\mathbb{E}_{XY}\left[f\left(X\right)g\left(Y\right)\right]-\mathbb{E}_{X}\left[f\left(X\right)\right]\mathbb{E}_{Y}\left[g\left(Y\right)\right]$
with $f,g$ being respectively functions in RKHS of $X$ and $Y$.

This can be interpreted as the generalization of the conventional
partial covariance. In what follows, conditional independence is achieved
if and only if the conditional cross-covariance is zero. For this
reason, in \cite{fukumizu2007kernel} the Hilbert-Schmidt norm of
the partial cross-covariance is tested against zero for the null hypothesis.

\subsection{Regression-based~methods}

Regression-based CI tests \cite{zhang2017causal,zhang2021testing,zhang2022residual}
assume that $Z$ is a confounder set of $X$ and $Y$, as well as
the relationships between $Z$ and $X$/$Y$ are additive noise models
(e.g., $X:=f\left(Z\right)+E,\;Z\indep E$). Therefore, by the use
of a suitable regression function, we can remove all the information
from $Z$ embeded in $X$/$Y$ by simply subtracting the regression
function, i.e., $r_{X}:=X-\hat{f}\left(Z\right)$ and $r_{Y}:=Y-\hat{g}\left(Z\right)$.
After this procedure, the conditional independence $X\indep Y|Z$
can be simplified to $r_{X}\indep r_{Y}$. Alternatively, in \cite{grosse2016identification}
the hypothesis testing problem $X\indep Y|Z$ is converted into $X-\hat{f}\left(Z\right)\indep\left(Y,Z\right)$.
Meanwhile, the problem is transformed into testing for $X-\hat{f}\left(Z\right)\indep\left(Y-\hat{g}\left(Z\right),Z\right)$
in \cite{zhang2017causal}.

Nonetheless, the assumptions required by regression-based methods
are relatively strong overall. As proved in \cite{hoyer2008nonlinear},
in general, we can only obtain independent residuals (w.r.t. $Z$)
if regression is performed against the true ``cause'' in the data
generating process, i.e., if data is generated as $Z:=f(X)+E,\;X\indep E$
then $X-\mathbb{E}\text{\ensuremath{\left[X|Z\right]}}$ typically
still depends on $Z$.

Additionally, the additive noise model assumption is also easy to
be violated. If the generating process involves non-additive noises
then the regression residuals can still be dependent on $Z$, making
the equivalence $X\indep Y|Z\Leftrightarrow r_{X}\indep r_{Y}$ invalid.

\subsection{Model-based~approaches}

Apart from the aforementioned methods, model-based approaches more
heavily borrow supervised and unsupervised learning algorithms as
the basis. For instance, recently GANs have been employed \cite{bellot2019conditional,shi2021double}
to implicitly learn to sample from the conditionals $p(X|Z)$ and
$p(Y|Z)$.

In \cite{bellot2019conditional} the main motivation is that for any
dependence measure $\rho$ and $\tilde{X}\sim p\left(X|Z\right)$
with $Y\indep\tilde{X}$, under both hypotheses, $\rho\left(X,Y,Z\right)\geq\rho\left(\tilde{X},Y,Z\right)$
and the equality occurs only under $\mathcal{H}_{0}$. This key observation
motivates the authors to employ GANs to learn the generator for $p\left(X|Z\right)$.
Then, the test's $p$-value can be empirically estimated by repeatedly
sampling $\tilde{X}$ and calculating the dependence measure without
any knowledge of the null distribution. Similarly, the double GAN
approach \cite{shi2021double} goes one step further--two generators
are used in order to learn both $p\left(X|Z\right)$ and $p\left(Y|Z\right)$,
then the test statistic is calculated as the maximum of the generalized
covariance measures of multiple transformation functions.

Differently from those, since conditional independence coincides with
zero conditional mutual information (CMI), which is a natural and
well-known measure of conditional dependence, many methods aim to
estimate CMI as the test statistic \cite{runge2018conditional,mukherjee2020ccmi,kubkowski2021gain}.
In \cite{runge2018conditional}, the CMI is approximated via several
$k$-nearest neighbors entropy estimators \cite{kozachenko1987sample}
and the test statistic is empirically estimated by randomly shuffling
samples of $X$ in a way that preserves $p\left(X|Z\right)$ while
breaking the conditional dependence between $X$ and $Y$ given $Z$.

As an extension to \cite{runge2018conditional}, since CMI estimators
are erroneous in high dimensions which is followed by the failed tests,
in \cite{kubkowski2021gain} CMI is replaced with the short expansion
of CMI, computed via the Möbius representation of CMI, which offers
simple asymptotic null distributions that allow for an easier construction
of the conditional independence test. Additionally, \cite{mukherjee2020ccmi}
propose an classifier-based estimator for the Kullback-Leibler divergence
to measure the divergence of $p\left(X,Y|Z\right)$ and $p\left(X|Z\right)\otimes p\left(Y|Z\right)$,
which is very closely related to CMI.

Following a slightly similar approach, the classification-based CI
test \cite{sen2017model} reduces the CI testing problem into a binary
classification problem, where the central idea is that under the null
hypothesis, a binary classifier cannot differentiate between the original
dataset and a shuffled dataset that forces the conditional independence
$X\indep Y|Z$; whereas the difference under the alternative hypothesis
can be easily captured by the classifier.

\subsection{Our approach}

Our \textbf{LCIT} method departs from regression-based methods in
the sense that it does not require any of the limited assumptions
supposed by these methods. More particularly, we devise a ``generalized
residual'', referred to as a ``latent representation'', such that
it is independent from $Z$ without assuming $Z$ is a confounder
of $X$ and $Y$ nor the relationships are additive noise models.

Moreover, while our method employs generative models, it approaches
the CI testing problem from an entirely different angle--instead
of learning to generate randomized samples as in GAN-based methods
\cite{bellot2019conditional,shi2021double}, we explicitly learn the
deterministic inner representations of $X$ and $Y$ through invertible
transformations of NFs, so that we can directly check for their independence
instead of adopting bootstrapping procedures.

%% file: methodology.tex
\inputencoding{latin9}Let $X,Y\in\mathbb{R}$ and $Z\in\mathbb{R}^{d}$
be our random variables and vectors where we wish to test for $X\indep Y|Z$
and $d$ is the number of dimensions of $Z$. In regard of $X$ and
$Y$ being limited to scalars instead of vectors, this is due to the
fact that in the majority of applications of conditional independence
testing, which include causal discovery tasks, we typically care about
the dependence of each pair of univariate variables given other variables.
Additionally, according to the Decomposition and Union rules of the
probabilistic graphoids \cite{pearl1986graphoids}, the conditional
independence between two sets of variables given the third set of
variables can be factorized into a series of conditional independencies
between pairs of univariate variables. Therefore, for simplicity,
this work focuses on real-valued $X$ and $Y$ exclusively.

\subsection{A New Characterization of Conditional Independence}

We start by giving the fundamental observation that drives the development
of our method:
\begin{lem}
\label{prop:Assuming-the-generative}Assuming the generative process
can be represented as $X=f\left(\epsilon_{X},Z\right)$ and $Y=g\left(\epsilon_{Y},Z\right)$
where $\left(\epsilon_{X},\epsilon_{Y}\right)\indep Z$ and $f,g$
are invertible functions w.r.t. their first argument, then
\begin{equation}
X\indep Y|Z\Leftrightarrow\epsilon_{X}\indep\epsilon_{Y}\label{eq:proxy-test}
\end{equation}
\end{lem}

\begin{proof}
Since $f$ and $g$ are invertible, by the change of variables rule
we have:

\begin{align*}
p\left(x|z\right) & =p\left(\epsilon_{X}|z\right)\left|\frac{\partial f}{\partial\epsilon_{X}}\right|^{-1}=p\left(\epsilon_{X}\right)\left|\frac{\partial f}{\partial\epsilon_{X}}\right|^{-1}\\
p\left(y|z\right) & =p\left(\epsilon_{Y}|z\right)\left|\frac{\partial g}{\partial\epsilon_{Y}}\right|^{-1}=p\left(\epsilon_{Y}\right)\left|\frac{\partial g}{\partial\epsilon_{Y}}\right|^{-1}\\
p\left(x,y|z\right) & =p\left(\epsilon_{X},\epsilon_{Y}|z\right)\left|\det\left(\begin{array}{cc}
\frac{\partial f}{\partial\epsilon_{X}} & \frac{\partial f}{\partial\epsilon_{Y}}\\
\frac{\partial g}{\partial\epsilon_{X}} & \frac{\partial g}{\partial\epsilon_{Y}}
\end{array}\right)\right|^{-1}\\
 & =p\left(\epsilon_{X},\epsilon_{Y}\right)\left|\frac{\partial f}{\partial\epsilon_{X}}\cdot\frac{\partial g}{\partial\epsilon_{Y}}\right|^{-1}
\end{align*}

where we ignore $z$ due to the constraint $\left(\epsilon_{X},\epsilon_{Y}\right)\indep Z$.

Thus, $p\left(x|z\right)p\left(y|z\right)=p\left(x,y|z\right)$ if
and only if $p\left(\epsilon_{X}\right)p\left(\epsilon_{Y}\right)=p\left(\epsilon_{X},\epsilon_{Y}\right)$,
rendering $\epsilon_{X}$ and $\epsilon_{Y}$ marginally independent.
\end{proof}
For this reason, to test for $X\indep Y|Z$ we can instead test for
$\epsilon_{X}\indep\epsilon_{Y}$, which is progressively less challenging
thanks to the reduced conditioning variables.

Furthermore, since $\epsilon_{X}$ and $\epsilon_{Y}$ act as contributing
factors to $X$ and $Y$ in the supposed generating process without
being observed, we refer to them as \textit{latent representations}
of $X$ and $Y$. This is inspired by regression based methods, where
the ``residuals'' are essentially the additive noises in the data
generation, and obtained by subtracting the regression functions to
remove all information from $Z$, i.e., $r_{X}:=X-\mathbb{E}\text{\ensuremath{\left[X|Z\right]}}$
and $r_{Y}:=Y-\mathbb{E}\text{\ensuremath{\left[Y|Z\right]}}$. As
explained in Section\textbf{~}\ref{sec:Background}, while residuals
are easy to compute, they cannot completely remove all information
from $Z$ as intended without restrictive conditions, including additive
noise models and $Z$ being a confounder set of both $X$ and $Y$.
Consequently, the dependence between $r_{X}$ and $r_{Y}$, which
is possibly caused by the remaining influences from $Z$ to $r_{X}$
and $r_{Y}$, may not exactly entail $X\not\indep Y|Z$. Therefore,
the \textit{latent variables} generalize and extend from residuals
at being truly independent from $Z$ by design.

We note that when $p\left(x|z\right)$ and $p\left(y|z\right)$ are
smooth and strictly positive then the use of the cumulative distribution
functions (CDF)--$\epsilon_{X}:=\mathrm{F}\left(x|z\right),\epsilon_{Y}:=\mathrm{F}\left(y|z\right),f:=\mathrm{F}^{-1}\left(\epsilon_{X}|z\right),$
and $g:=\mathrm{F}^{-1}\left(\epsilon_{X}|z\right)$--is naturally
a candidate for \eqref{eq:proxy-test}. Notably, the cumulative distribution
functions for the alternative unconditional test has also been employed
in \cite{petersen2021testing}, where quantile regression is used
to estimate the cumulative distribution functions. However, in this
study, Lemma 1 emphasizes the application of a more generic invertible
transformation that is not restricted to CDFs, which can be parametrized
with NFs.

\subsection{Conditional Normalizing Flows for Latent Variable Inference}

In this sub-section we explain in detail the conditional normalizing
flows (CNF) models used to infer the latent representations of $X$
and $Y$ conditional on $Z$, which will be used for the proxy unconditional
test as in \eqref{eq:proxy-test}. For brevity, since the models for
$X$ and $Y$ are identically implemented except for their learnable
parameters, we only describe the CNF module for $X$ and the module
for $Y$ follows accordingly.

\subsubsection{Unconditional Normalizing Flows Modeling}

Normalizing Flows have been progressively developed in the last several
years since they were defined in \cite{tabak2010density} and made
popular in \cite{rezende2015variational,dinh2014nice}. However, many
of NF methods mainly concern with high dimensional data where interactions
between dimensions are required for the invertible transformation
to be possible, for example \cite{dinh2014nice,dinh2016density,kingma2016improved,papamakarios2017masked,kingma2018glow}.

Conversely, in this study we are interested in transformations of
one dimensional (scalar) variables. That being said, there also exists
NFs for scalars, for example, CDF-based flows, mixture of logistics
flows \cite{papamakarios2021normalizing}, splines \cite{muller2019neural},
or nonlinear squared flow \cite{ziegler2019latent}. For \textbf{LCIT},
in particular, we demonstrate the usage of CDF flows thanks to their
simplicity with no constraint. However, it should be noted that any
valid alternative NFs for scalar variables can be naturally adapted
into our solution.

The essence of CDF flows begins with the observation that the CDF
of any strictly positive density function (e.g., Gaussian, Laplace,
or Student's t distributions) is differentiable and strictly increasing,
hence the positively weighted combination of any set of these distributions
is also a differentiable and strictly increasing function, which entails
invertibility. This allows us to parametrize the flow transformation
in terms of a mixture of Gaussians \cite{papamakarios2021normalizing},
which is known as a universal approximator of any smooth density \cite{goodfellow2016deep}.
Therefore, we can approximate any probability density function to
arbitrary non-zero precision, given a sufficient number of components.

More concretely, starting with an unconditional Gaussian mixture based
NF, we denote $k$ as the number of components, along with $\left\{ \mu_{i},\sigma_{i}\right\} _{i=1}^{k}$
and $\left\{ w_{i}\right\} _{i=1}^{k}$ as the parameters and the
weight for each component in the mixture of univariate Gaussian densities
$p\left(x\right)=\sum_{i=1}^{k}w_{i}\mathcal{N}\left(x;\mu_{i},\sigma_{i}^{2}\right)$,
where $w_{i}\geq0$, $\sum_{i=1}^{k}w_{i}=1$. Subsequently, the invertible
mapping $\mathbb{R}\mapsto\left(0,1\right)$ is defined as

\begin{equation}
u\left(x\right)=\sum_{i=1}^{k}w_{i}\Phi_{i}\left(x\right)\label{eq:transform}
\end{equation}

where $\Phi_{i}\left(x\right)$ is the cumulative distribution function
of the $i$-th Gaussian component, i.e., $\Phi_{i}\left(x\right)=\int_{\infty}^{x}\mathcal{N}\left(t;\mu_{i},\sigma_{i}^{2}\right)dt$.

Along with the fact that $u\in\left(0,1\right)$, this transformation
entails that $u$ has a standard uniform distribution:

\begin{align}
p\left(u\right) & =p\left(x\right)\left|\frac{\partial u}{\partial x}\right|^{-1}\nonumber \\
 & =p\left(x\right)\left|\sum_{i=1}^{k}w_{i}\frac{\partial\Phi_{i}}{\partial x}\right|^{-1}\nonumber \\
 & =p\left(x\right)\left|\sum_{i=1}^{k}w_{i}\mathcal{N}\left(x;\mu_{i},\sigma_{i}^{2}\right)\right|^{-1}\nonumber \\
 & =p\left(x\right)\left|p\left(x\right)\right|^{-1}\nonumber \\
 & =1\label{eq:uniform}
\end{align}

Following this reasoning, the transformation \eqref{eq:transform}
maps a variable $x$ with Gaussian mixture density, which is multimodal
and complex, into a simpler standard uniform distribution. Moreover,
due to its monotonicity, it is also possible to reverse the process
to generate new $x$ after sampling $u$, though it is not necessary
in the considering application of CI testing.

\RestyleAlgo{ruled}

\begin{algorithm}[t]
\caption{The latent inference algorithm.\label{alg:Training-algorithm}}

\textbf{Input:} Empirical samples $\mathcal{D}=\left\{ \left(x_{i},z_{i}\right)\right\} _{i=1}^{n}$
of $\left(X,Z\right)$ as well as additional hyper-parameters for
training, for example the learning rate $\eta$.

\textbf{Output:} The corresponding latents $\epsilon_{X}$ of $X$
given $Z$.
\begin{enumerate}
\item Initialize neural networks $\mathrm{MLP}_{\mu}$, $\mathrm{MLP}_{\log\sigma^{2}}$,
and $\mathrm{MLP}_{w}$. Denote the union set of their parameters
as $\theta$.
\item Repeat until converge:
\begin{enumerate}
\item Sample a mini-batch $\mathcal{B}$ of $b$ samples from $\mathcal{D}$.
\item Compute the mini-batch means, variances, and weights of the conditional
Gaussian mixtures:

\begin{align*}
\mu & :=\mathrm{MLP}_{\mu}\left(z_{\mathcal{B}}\right)\\
\sigma^{2} & :=\exp\left(\mathrm{MLP}_{\log\sigma^{2}}\left(z_{\mathcal{B}}\right)\right)\\
w & :=\mathrm{MLP}_{w}\left(z_{\mathcal{B}}\right)
\end{align*}

\item Compute the mini-batch log-likelihood:

\[
\mathcal{L}:=\frac{1}{b}\sum_{i\in\mathcal{B}}\log w_{i}\cdot\mathcal{N}\left(x_{i};\mu_{i},\sigma_{i}^{2}\right)
\]

\item Update parameters using gradient ascent:

\[
\theta:=\theta+\eta\nabla_{\theta}\mathcal{L}
\]

\end{enumerate}
\item Infer the latents for the whole dataset $\mathcal{D}$:

\begin{align*}
\mu & :=\mathrm{MLP}_{\mu}\left(z_{\mathcal{D}}\right)\\
\sigma^{2} & :=\exp\left(\mathrm{MLP}_{\log\sigma^{2}}\left(z_{\mathcal{D}}\right)\right)\\
w & :=\mathrm{MLP}_{w}\left(z_{\mathcal{D}}\right)\\
u & :=w\cdot\Phi\left(x_{\mathcal{D}};\mu,\sigma^{2}\right)\\
\epsilon_{X} & :=\Phi^{-1}\left(u\right)
\end{align*}

\item Return $\epsilon_{X}$.
\end{enumerate}
\end{algorithm}

\subsubsection{From Unconditional to Conditional Normalizing Flows}

Next, to extend unconditional NF to conditional NF, we parametrize
the weights, means, as well as variances of the Gaussian components
as functions of $z$ using neural networks. To be more specific, we
parametrize $\mu\left(z\right)\in\mathbb{R}^{k}$ as a simple Multiple
Layer Perceptron (MLP) with real-valued outputs. As for $\sigma^{2}$,
since it is constrained to be positive, we instead model $\log\sigma^{2}\left(z\right)\in\mathbb{R}^{k}$
with an MLP, similarly as $\mu\left(z\right)$. Finally, for the weights,
which must be non-negative with sum of one, $w_{i}\left(z\right)\in\left(0,1\right)^{k}$
is parametrized with an MLP with the Softmax activation function for
the last layer. These steps can be summarized as follows:

\[
u\left(x,z\right):=\sum_{i=1}^{k}w_{i}\left(z\right)\Phi_{i}\left(x|z\right)
\]

It is worth emphasizing that this translation naturally preserves
the property similarly to \eqref{eq:uniform}. In another word, $u\left(x,z\right)\in\left(0,1\right)$
and

\begin{align*}
p\left(u|z\right) & =p\left(x|z\right)\left|\frac{\partial u}{\partial x}\right|^{-1}=1\\
p\left(u\right) & =\int p\left(u|z\right)p\left(z\right)dz=\int p\left(z\right)dz=1
\end{align*}

Therefore, $U$ is both conditionally and marginally standard uniform
regardless of the value of $z$, making $U$ unconditionally independent
of $Z$.

Moreover, we introduce an additional flow that depends on $u$ only.
This flow adopts the inverse cumulative distribution function (iCDF)
of the standard Gaussian distribution, to transform $u$ from a standard
uniform variable to a standard Gaussian variable, which is the final
``latent'' variable we use for the surrogate test in \eqref{eq:proxy-test}:

\[
\epsilon\left(u\right):=\Phi^{-1}\left(u\right)
\]

\RestyleAlgo{ruled}

\begin{algorithm}[t]
\caption{The proposed \textbf{L}atent representation based \textbf{C}onditional
\textbf{I}ndependence \textbf{T}esting (\textbf{LCIT}) algorithm.\label{alg:LCIT}}

\textbf{Input:} Empirical samples $\mathcal{D}=\left\{ \left(x_{i},y_{i},z_{i}\right)\right\} _{i=1}^{n}$
of $\left(X,Y,Z\right)$ and the significance level $\alpha$.

\textbf{Output:} The $p$-value and whether $X\indep Y|Z$ or not.
\begin{enumerate}
\item Use Algorithm~\ref{alg:Training-algorithm} to infer the latents
$\epsilon_{X}$ and $\epsilon_{Y}$.
\item Calculate the test statistic and $p$-value:

\begin{align*}
r & :=\frac{\mathrm{cov}\left(\epsilon_{X},\epsilon_{Y}\right)}{\sigma_{\epsilon_{X}}\sigma_{\epsilon_{Y}}}\\
t & :=\frac{1}{2}\ln\frac{1+r}{1-r}\\
p\text{-value} & :=2\left(1-\Phi\left(\left|t\right|\sqrt{n-3}\right)\right)
\end{align*}

\item Return $p$-value and

\[
\text{Conclusion}:=\begin{cases}
X\indep Y|Z & \text{if }p\text{-value}>\alpha\\
X\not\indep Y|Z & \text{if }p\text{-value \ensuremath{\leq\alpha}}
\end{cases}
\]
\end{enumerate}
\end{algorithm}

\subsubsection{Learning Conditional Normalizing Flows}

Similarly to conventional NFs, we adopt the Maximum Likelihood Estimation
(MLE) framework to learn the CNF modules for $X$ and $Y$.

With a fixed number of components $k$, we denote $\theta$ as the
total set of parameters to be learned, which includes the parameters
of the neural networks $\mu\left(z\right)$, $\log\sigma^{2}\left(z\right)$,
and $w\left(z\right)$. Then, the conditional likelihood of $X$ is
given by

\begin{align*}
p_{\theta}\left(x|z\right) & =\sum_{i=1}^{k}w_{i}\left(z\right)\mathcal{N}\left(x;\mu_{i}\left(z\right),\sigma_{i}^{2}\left(z\right)\right)
\end{align*}

and we aim to maximize the log-likelihood of observed $X$ conditioned
on $Z$ over the space of $\theta$:

\begin{align*}
\mathcal{L}\left(\theta\right) & :=\frac{1}{n}\sum_{i=1}^{n}\log p_{\theta}\left(x_{i}|z_{i}\right)\\
\theta_{\text{MLE}}^{*} & :=\argmax_{\theta}\mathcal{L}\left(\theta\right)
\end{align*}

where $\left\{ \left(x_{i},z_{i}\right)\right\} _{i=1}^{n}$ is the
set of $n$ observed samples of $\left(X,Z\right)$.

Subsequently, a gradient based optimization framework can be applied
to learn $\theta$ since $\mathcal{L}$ is fully differentiable. In
summary, Algorithm~\ref{alg:Training-algorithm} highlights the main
steps of the training process.

\subsection{Marginal Independence Test for the Latents}

For two jointly Gaussian variables\footnote{While our procedure only constrains $\left(\epsilon_{X},\epsilon_{Y}\right)$
to be marginally Gaussian, the experiments show that it is still robust
in a wide range of scenarios.}, their independence is equivalent to zero Pearson's correlation coefficient.
Therefore, we can resort to Fisher's transformation to get the test
statistic, which has approximate Gaussian distribution under the null
hypothesis, hence the closed form of the $p$-value is available \cite{hotelling1953new}.

To be more specific, first we calculate the Pearson's correlation
coefficient $r$ between $\epsilon_{X}$ and $\epsilon_{Y}$, then
turn it into the test statistic $t$ using the Fisher's transformation:

\begin{align*}
r & :=\frac{\mathrm{cov}\left(\epsilon_{X},\epsilon_{Y}\right)}{\sigma_{\epsilon_{X}}\sigma_{\epsilon_{Y}}}\\
t & :=\frac{1}{2}\ln\frac{1+r}{1-r}
\end{align*}

where the test statistic $t$ has an approximate Gaussian distribution
with mean $\frac{1}{2}\ln\frac{1+\rho}{1-\rho}$ and standard deviation
of $\frac{1}{\sqrt{n-3}}$ where $\rho$ is the true correlation between
$\epsilon_{X}$ and $\epsilon_{Y}$, and $n$ is the sample size.

Therefore, under the null hypothesis where $\rho=0$, the $p$-value
for $t$ can be calculated as the two-tail extreme region:

\[
p\text{-value}:=2\left(1-\Phi\left(\left|t\right|\sqrt{n-3}\right)\right)
\]

Finally, with a significance level of $\alpha$, the null hypothesis
$\mathcal{H}_{0}$ is rejected if $p\text{-value}<\alpha$, otherwise
we fail to reject $\mathcal{H}_{0}$ and have to conclude conditional
dependence $X\not\indep Y|Z$. To summarize, see Algorithm~\ref{alg:LCIT}
for the whole computation flow of the proposed \textbf{LCIT} method.

\begin{figure}[t]
\centering{}\includegraphics[width=0.95\columnwidth]{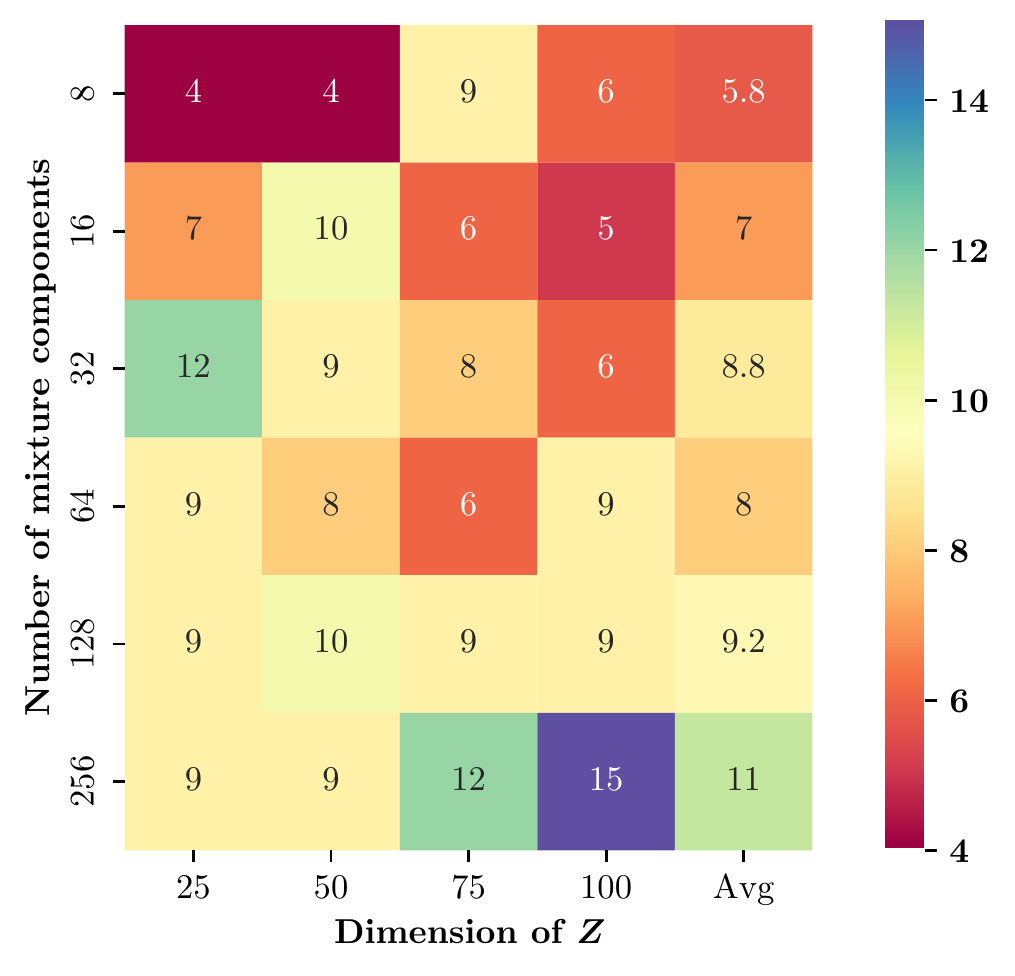}\caption{Hyper-parameter tuning results for the number of mixture components.
The number in each cell denotes the number of runs that have the respective
number of components as the best setting recorded. Each column sums
to $50$ independent runs.\label{fig:n_components}}
\end{figure}

\begin{figure}[t]
\centering{}\includegraphics[width=0.95\columnwidth]{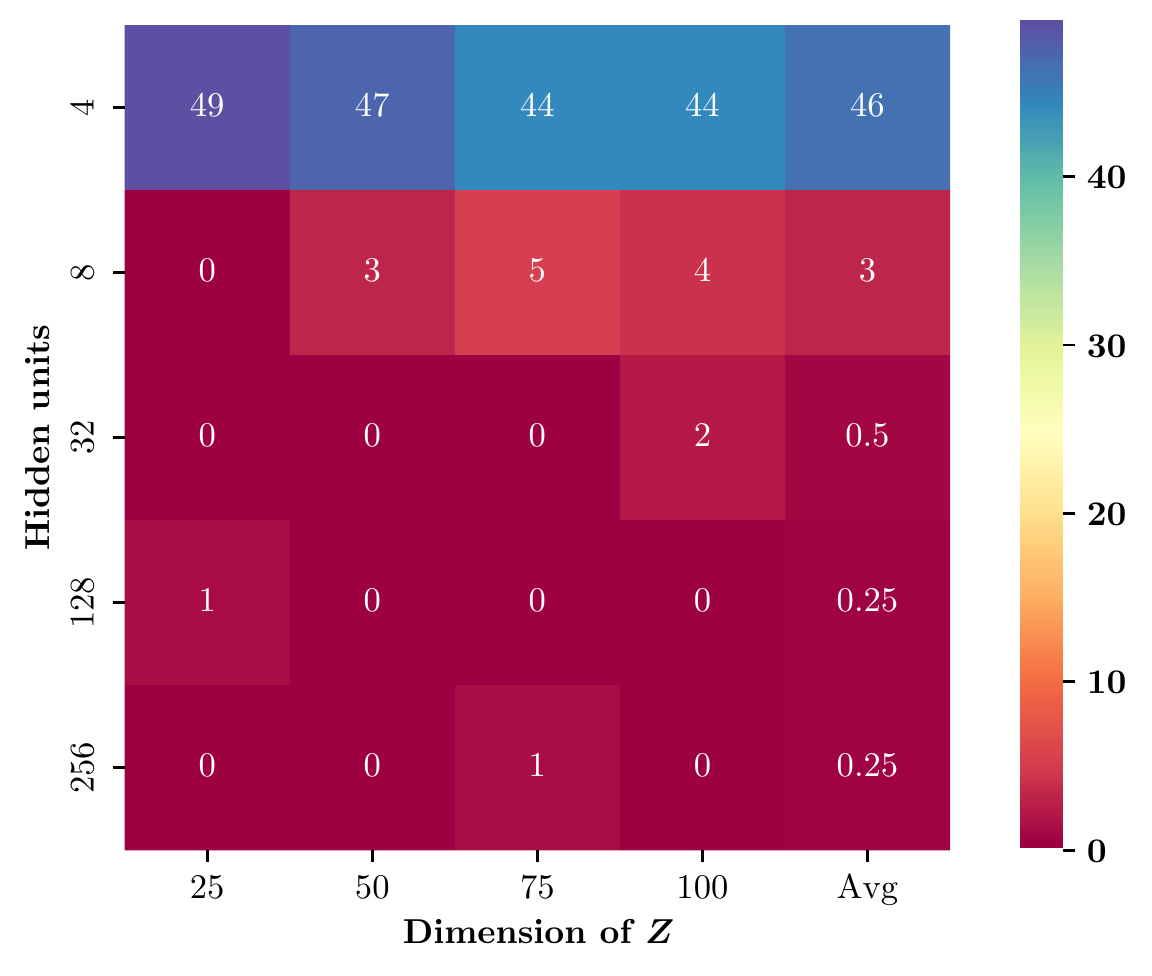}\caption{Hyper-parameter tuning results for the hidden layer size. The number
in each cell denotes the number of runs that have the respective number
of hidden units as the best setting recorded. Each column sums to
$50$ independent runs.\label{fig:hidden-size}}
\end{figure}

\begin{figure*}[t]
\centering{}\includegraphics[width=1\textwidth]{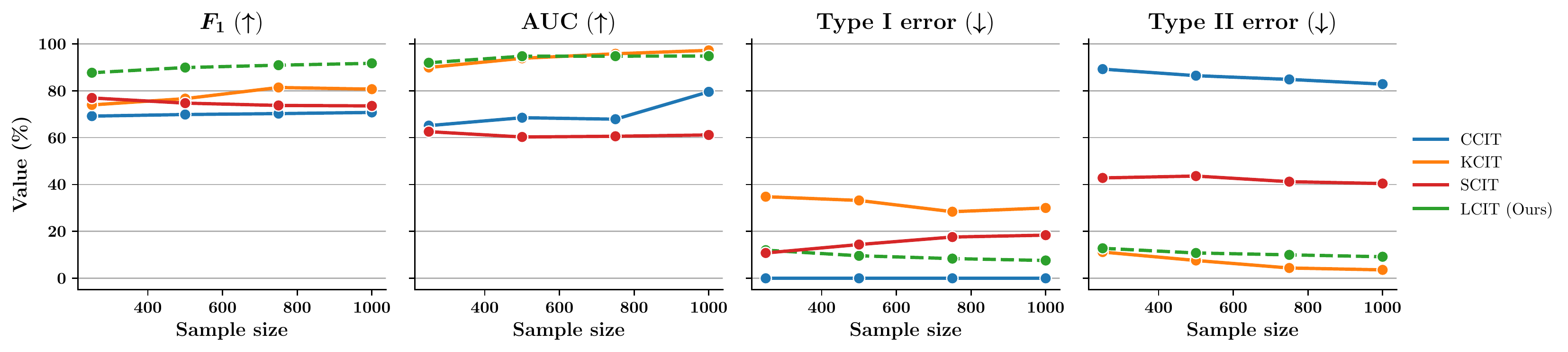}\caption{Conditional Independence Testing performance as a function of sample
size. The evaluation metrics are $F_{1}$ score, AUC (higher is better),
Type I and Type II error rates (lower is better). We compare our proposed
\textbf{LCIT} method with three baselines CCIT \cite{sen2017model},
KCIT \cite{zhang2012kernel}, and SCIT \cite{zhang2022residual}.
Our method achieves the best $F_{1}$ and AUC scores, and is the only
method that retains good scores in all metrics and sample sizes.\label{fig:sample-size}}
\end{figure*}

\begin{figure*}[t]
\centering{}\includegraphics[width=1\textwidth]{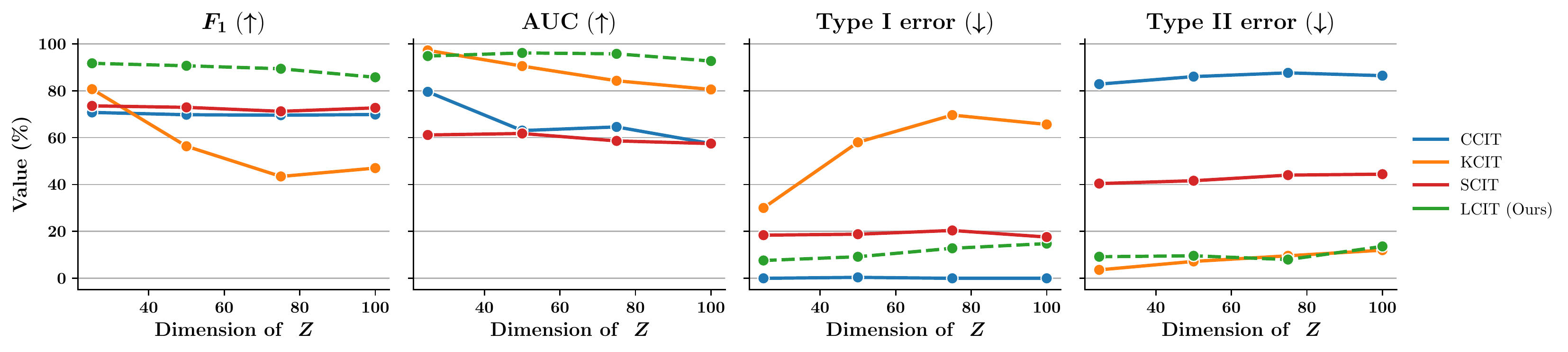}\caption{Conditional Independence Testing performance as a function of dimensionality.
The evaluation metrics are $F_{1}$ score, AUC (higher is better),
Type I and Type II error rates (lower is better). We compare our proposed
\textbf{LCIT} method with three baselines CCIT \cite{sen2017model},
KCIT \cite{zhang2012kernel}, and SCIT \cite{zhang2022residual}.
Our method achieves the best $F_{1}$ and AUC scores, and is the only
method that retains good scores in all metrics and dimensionalities.\label{fig:dimension}}
\end{figure*}

%% file: experiments.tex
\inputencoding{latin9}\begin{figure}[t]
\centering{}\includegraphics[width=1\columnwidth]{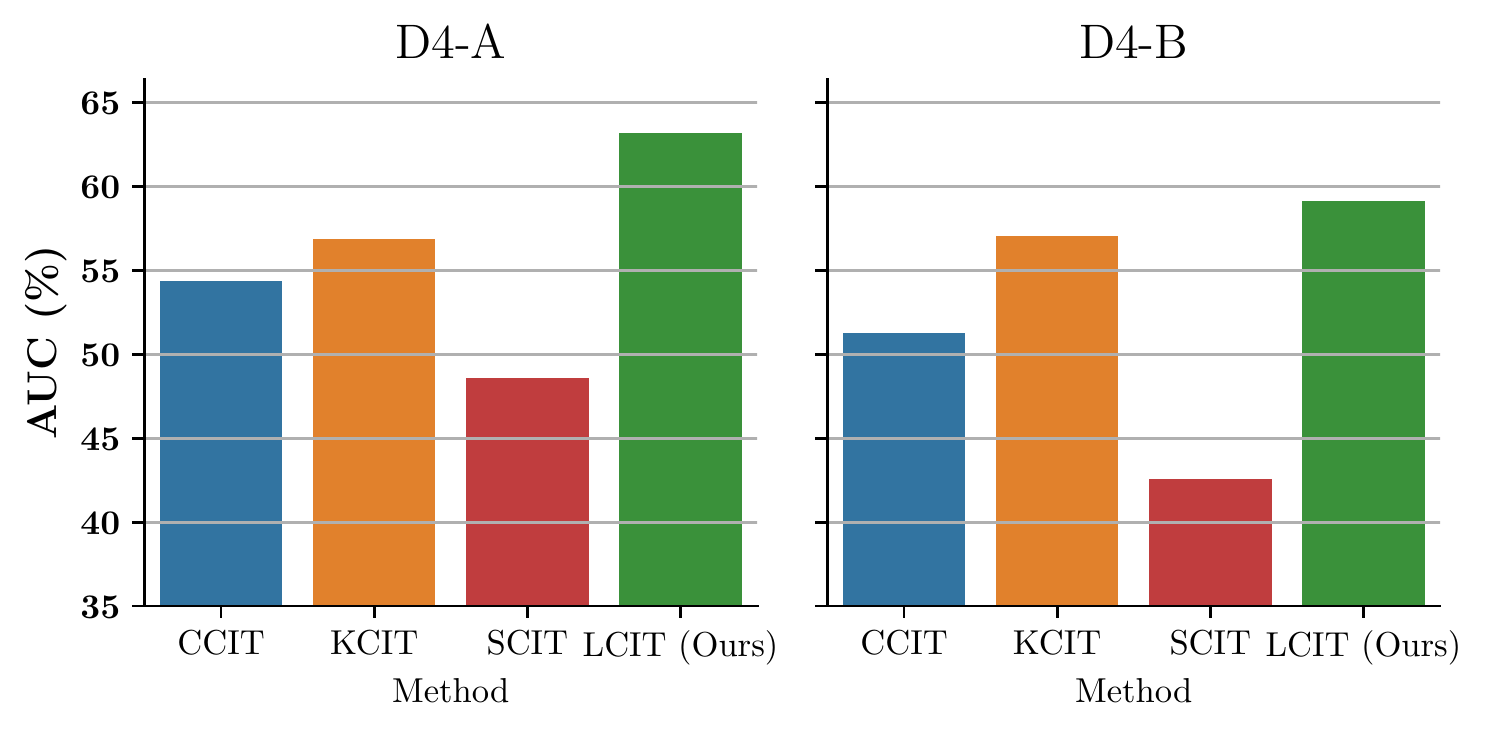}\caption{Conditional Independence Testing performance on real datasets. The
evaluation metrics is AUC (higher is better). We compare our proposed
\textbf{LCIT} method with three baselines CCIT \cite{sen2017model},
KCIT \cite{zhang2012kernel}, and SCIT \cite{zhang2022residual}.
\textbf{LCIT} surpasses all other methods, especially SCIT in large
margins.\label{fig:real-data}}
\end{figure}

\subsection{Setup}

\subsubsection{Methods}

We demonstrate the effectiveness of the \textbf{LCIT} in conditional
independence testing tasks on synthetic and real data against popular
and recent state-of-the-art methods across different approaches. More
specifically, we consider the Kernel-based KCIT\footnote{We use the KCIT implementation from the CMU causal-learn package:
\url{https://github.com/cmu-phil/causal-learn}} method \cite{zhang2012kernel} as a popular competitor, the recently
emerged residual similarity based SCIT\footnote{We follow the authors' original source code in Matlab: \url{https://github.com/Causality-Inference/SCIT}}
approach \cite{zhang2022residual}, and lastly the classification
based CCIT\footnote{We use the implementation from original authors: \url{https://github.com/rajatsen91/CCIT}}
algorithm \cite{sen2017model}. Regarding the configurations, we use
the default parameters recommended by respective baseline methods.

Additionally, for all methods, we first apply a standard normalization
step for each of $X$, $Y$, and $Z$ before performing the tests.

Specifically for \textbf{LCIT}, we parametrize $\mathrm{MLP}_{\mu}$,
$\mathrm{MLP}_{\log\sigma^{2}}$, and $\mathrm{MLP}_{w}$ with neural
networks of one hidden layer, Rectifier Linear Unit (ReLU) and Batch-normalization
\cite{ioffe2015batch} activation functions. The hyper-parameters
specifications and analyses are considered in the followings.

\subsubsection{Training CNFs}

We implement the CNF modules using the PyTorch framework \cite{paszke2019pytorch}.
Each CNF module is trained using the Adam optimization algorithm with
fixed learning rate of $5\times10^{-3}$, weight decay of $5\times10^{-5}$
to help regularize the model's complexity, and batch size of $64$.

\subsubsection{Data processing}

Before feeding data into the CNF modules, apart from data standardization
as other methods, we additionally handle outliers by clipping the
data to be in between the $2.5\%$ and $97.5\%$ quantiles of each
dimension, which helps stabilize the training process since neural
networks can be very sensitive to extreme values.

Subsequently, the input dataset is divided into training and validation
sets with a ratio of $70/30$, where the training portion is used
to learn parameters and the validation set enables early stopping.
We notice that typically, the training process only requires under
$20$ training epochs.

\subsubsection{Hyper-parameters}

The most important hyper-parameters presented in our models are the
number of Gaussian mixture components and the hidden layer sizes,
which together determine the expressiveness of the CNF. To examine
which set of configuration works best, we furthermore perform hyper-parameter
tuning on these two factors using the Optuna framework \cite{akiba2019optuna}.

Specifically, we simulate $50$ random datasets in the same manner
as in sub-section \ref{subsec:Synthetic-data}, both CI and non-CI
labels, for each number of dimensions of $Z$ varying in $\left\{ 25,50,75,100\right\} $.
With each dataset we execute $20$ Optuna trials to find the best
configuration, where the number of components varies in $\left\{ 8,16,32,64,128,256\right\} $
and the number of units in the hidden layer takes value in $\left\{ 4,8,16,32,64,128,256\right\} $.
The objective is the sum of the maximum log-likelihoods of $X$ and
$Y$ conditioned on $Z$. Finally, the most advantageous setting by
far is recorded.

The summary of best configurations found via hyper-parameter tuning
is showed in Figure~\ref{fig:n_components} and Figure~\ref{fig:hidden-size}.
It is clear that the more mixture components is usually preferable,
while a small number of hidden units is sufficiently effective to
retain a high performance. Based on these, we fix $32$ Gaussian components
for all other experiments in order to keep computations low with only
a small loss of performance compared with larger numbers of components.
On another hand, four hidden units is also used for every experiments
since it is both computationally efficient and sufficiently expressive.

\subsection{Synthetic data\label{subsec:Synthetic-data}}

Following several closed related studies \cite{zhang2012kernel,doran2014permutation,strobl2019approximate,bellot2019conditional},
we randomly simulate the datasets following the post-nonlinear additive
noise model:

\begin{align*}
\mathcal{H}_{0} & :Z:=f\left(X\otimes a+E_{f}\right),\quad Y:=g\left(\left\langle Z,b\right\rangle +E_{g}\right)\\
\mathcal{H}_{1} & :Z:=f\left(X\otimes a+E_{f}\right),\quad Y:=g\left(\left\langle Z,b\right\rangle +cX+E_{g}\right)
\end{align*}

with $X:=2E_{X}$ where $E_{X}$, $E_{f}$, and $E_{g}$ are independent
noise variables following the same distribution randomly selected
from $\left\{ \mathcal{U}\left(-1,1\right),\mathcal{N}\left(0,1\right),\mathrm{Laplace}\left(0,1\right)\right\} $.
The $\otimes$ and $\left\langle \cdot,\cdot\right\rangle $ denote
the outer and inner products, respectively. Additionally, $a,b\sim\mathcal{U}\left(-1,1\right)^{d}$,
$c\sim\mathcal{U}\left(1,2\right)$, and $f,g$ are uniformly chosen
from a rich set of mostly non-linear functions $\left\{ \alpha x,x^{2},x^{3},\tanh x,x^{-1},e^{-x},\frac{1}{1+e^{-x}}\right\} $\footnote{The input is appropriately scaled and translated before being fed
into each function.}.

\subsubsection{Effect of different sample sizes}

To study the performance of the \textbf{LCIT} against alternative
tests across different sample sizes, we fix the dimensionality $d$
of $Z$ at 25 and vary the sample size from 250 to 1,000. We measure
the CI testing performance under four different metrics, namely the
$F_{1}$ score (higher is better), Area Under the Receiver Operating
Characteristic Curve (AUC, higher is better), as well as Type I and
Type errors II (lower is better). More specifically, Type I error
refers to the proportion of false rejections under $\mathcal{H}_{0}$,
and Type II error reflects the proportion of false acceptances under
$\mathcal{H}_{1}$. These metrics are evaluated using 250 independent
runs for each combination of method, sample size, and label. Additionally,
for $F_{1}$ score, Type I, and Type II errors, we adopt the commonly
used significance level of $\alpha=0.05$.

The result is reported in Figure~\eqref{fig:sample-size}, which
shows that our method is the only one achieving good and stable performance
in all four evaluation criterions. Remarkably, \textbf{LCIT} scores
the highest in terms of $F_{1}$ measure, surpassing all other tests
with clear differences. Furthermore, \textbf{LCIT}, along with KCIT,
also obtains the highest AUC scores, approaching closely to $100\%$
as more samples are used, and marginally outperforms the two recent
state-of-the-arts CCIT and SCIT at considerable margins. For both
Type I and II errors, \textbf{LCIT} stably earns the second lowest
in overall, at around $10\%$.

Meanwhile, KCIT achieves the lowest Type II errors, but its Type I
errors are completely larger than those of all other methods, suggesting
that KCIT majorly returns conditional dependence as output. In contrary,
while CCIT is able to obtain virtually no error in Type I, its use
becomes greatly unreliable when viewed from the perspective of Type
II errors. This indicates that CCIT in general usually favors outputting
conditional independence as the answer.

\subsubsection{Effect of high dimensional conditioning sets}

In Figure~\eqref{fig:dimension} we study the change in performance
of \textbf{LCIT} as well as other methods in higher dimensional settings.
Concretely, we fix the sample size at 1,000 samples and increase the
dimension of $Z$ from 25 to 100. The result shows that our method
consistently outperforms other state-of-the-arts as the dimensionality
of $Z$ increases, as evidenced by the highest AUC scores in overall,
leaving CCIT and SCIT by up to roughly 40 units, while having comparably
low error rates.

Generally, we can see a visible decline in performance of all methods,
especially KCIT and CCIT. The AUC score of KCIT drops rapidly by 20
units from the smallest to largest numbers of dimensions, whereas
its $F_{1}$ score deteriorates quickly to half of the initial value,
and the Type I errors are always the highest among all considered
algorithms. On another hand, CCIT also has vanishing Type I errors
but exceedingly high Type II errors similarly to that in Figure~\eqref{fig:sample-size}.

\subsection{Real data}

To furthermore demonstrate the robustness of the proposed \textbf{LCIT}
test, we evaluate it against other state-of-the-arts in CI testing
on real datasets.

In general, real datasets of triplets $\left(X,Y,Z\right)$ for CI
test benchmarking are not available, so we have to resort to data
coming with ground truth networks instead, which are still relatively
rare and there are few consensus benchmark datasets.

In this study, we examine two datasets from the Dialogue for Reverse
Engineering Assessments and Methods challenge\footnote{\url{https://dreamchallenges.org}},
ninth edition (referred to as DREAM4) \cite{marbach2009generating,marbach2010revealing},
where the data is publicly accessible with the ground-truth gene regulatory
networks presented. The challenge's objective is to recover the gene
regulatory networks given their gene expression data only. Therefore,
the data sets are well fit to the application of our method and CI
tests in general.

Regarding data description, each data set includes a ground truth
transcriptional regulatory network of \textit{Escherichia coli} or
\textit{Saccharomyces cerevisiae}, along with observations of gene
expression measurements. We denote the two considered datasets as
D4-A and D4-B, where D4-A is from the first sub-challenge of the contest
that contains $10$ genes with $105$ gene expression observations,
whereas D4-B comes from the second sub-challenge and consists of $100$
genes complemented with $210$ gene expression samples.

Next, we extract conditional independent and conditional dependent
triplets $\left(X,Y,Z\right)$ from the ground truth networks. This
process is done based on the fact that if there is a direct connection
between two nodes in a network, then regardless of the conditioning
set, they remain conditionally dependent. Otherwise, the union of
their parent sets should d-separate all paths connecting them, rendering
them conditionally independent given the joint parents set \cite{pearl1998graphical}.
To create class-balance datasets, for D4-A, we collect $30$ conditional
independence and $30$ conditional dependence relationships, while
the number of relationships from D4-B are $50$ each.

The CI testing performance on the DREAM4 datasets is reported in Figure~\ref{fig:real-data}.
It can be seen that the results follow relatively consistently with
synthetic data scenarios, with \textbf{LCIT} being the best performer,
followed by KCIT. Meanwhile, CCIT and SCIT considerably underperform
with AUC scores around or under $50\%$, comparable to a fair-coin
random guesser.

%% file: conclusion.tex
\inputencoding{latin9}In this paper we propose a representation learning
approach to conditional independence testing called \textbf{LCIT}.
Through the use of conditional normalizing flows, we transform the
difficult conditional independence testing problem into an easier
unconditional independence testing problem. We showcase the performance
of our \textbf{LCIT} method via intensive experiments including synthetic
datasets of highly complex relationships, as well as real datasets
in bio-genetics. The empirical results show that \textbf{LCIT} performs
really well and is able to consistently outperform existing state-of-the-arts.

Conditional independence testing is a generic tool that serves as
the basis of a wide variety of scientific tasks, especially in causal
discovery. Therefore, the development of \textbf{LCIT} offers a promising
generic alternative solution for these problems and methods.

As for future perspectives, since the latent representation based
approach is first used in \textbf{LCIT}, it opens doors for further
scientific developments of conditional independence tests based on
representation learning, which are expected to greatly improve from
\textbf{LCIT} and are able to extend to more challenging scenarios
such as heterogeneity and missing data.